\newif\ifisarxiv
\isarxivtrue
\newif\ifisneurips
\isneuripstrue

\ifisarxiv
\documentclass[11pt]{article}
\usepackage[margin=1.25in]{geometry}
\usepackage[numbers]{natbib}
\else
\documentclass{article}
\PassOptionsToPackage{numbers}{natbib}
\usepackage[preprint]{sty/neurips2019/neurips_2019}
\fi

\usepackage{hyperref}       

\usepackage[utf8]{inputenc} 
\usepackage[T1]{fontenc}    
\usepackage{url}            
\usepackage{booktabs}       
\usepackage{amsfonts}       
\usepackage{nicefrac}       
\usepackage{microtype}      
\usepackage{amsmath,amssymb}
\usepackage{color}
\usepackage{graphicx}
\usepackage{multirow}
\usepackage{wrapfig}
\usepackage{algorithm}
\usepackage{algorithmic}
\usepackage{xfrac}
\usepackage{caption}
\usepackage{enumitem}
\usepackage{cleveref}
\usepackage{todonotes}
\usepackage{xspace}

\def\Vol{{\mathrm{R\textnormal{-}DPP}}}
\def\DPP{{\mathrm{DPP}}}

\def\cmark{\Green{\checkmark}}
\def\xmark{\Red{\large\sffamily x}}

\def\poly{{\mathrm{poly}}}
\def\polylog{{\mathrm{polylog}}}
\def\simiid{\overset{\textnormal{\fontsize{6}{6}\selectfont
i.i.d.}}{\sim}}
\def\tinydots{\textnormal{\fontsize{6}{6}\selectfont \dots}}
\def\tinycdots{\textnormal{\fontsize{6}{6}\selectfont $\cdots$}}

\def\checkmark{\tikz\fill[scale=0.4](0,.35) -- (.25,0) --
(1,.7) -- (.25,.15) -- cycle;} 

\newcommand{\deff}{d_{\textnormal{eff}}}
\newcommand{\transp}{\top}

\newcommand{\normsmall}[1]{\Vert #1 \Vert}

\newcommand{\CommaBin}{\mathbin{\raisebox{0.5ex}{,}}}

\def\sigmat{\widetilde{\sigma}}

\newcommand{\cO}{\mathcal{O}}
\newcommand{\tcO}{\widetilde{\cO}}

\def\Lb{\mathbf{L}}
\def\L{\mathbf{L}}

\def\Lbh{\widehat{\mathbf{L}}}
\def\Lbt{\widetilde{\mathbf{L}}}
\def\Lt{\widetilde{L}}

\def\Ic{\mathcal{I}}

\def\R{\mathbf R}

\ifx\BlackBox\undefined
\newcommand{\BlackBox}{\rule{1.5ex}{1.5ex}}  
\fi

\DeclareMathOperator*{\argmax}{\mathop{\mathrm{argmax}}}

\def\b{\mathbf b}
\def\bt{\widetilde{\mathbf b}}

\def\e{\mathbf e}

\def\one{\mathbf 1}

\ifisneurips
   \def\st{z}
   \def\sh{s}
\else
   \def\st{\tilde{s}}
   \def\sh{\hat{s}}
\fi

\def\ee{\mathrm{e}}

\def\X{\mathbf X}

\def\B{\mathbf B}
\def\Bt{\widetilde{\mathbf B}}

\def\St{\widetilde{S}}

\def\I{\mathbf I}

\def\P{\mathbf P}

\def\E{\mathbb E}

\def\R{\mathbb R} 
\def\tr{\mathrm{tr}}

\def\rank{\mathrm{rank}}

\let\origtop\top
\renewcommand\top{{\scriptscriptstyle{\origtop}}} 

\definecolor{silver}{cmyk}{0,0,0,0.3}
\definecolor{yellow}{cmyk}{0,0,0.9,0.0}
\definecolor{reddishyellow}{cmyk}{0,0.22,1.0,0.0}
\definecolor{black}{cmyk}{0,0,0.0,1.0}
\definecolor{darkYellow}{cmyk}{0.2,0.4,1.0,0}
\definecolor{darkSilver}{cmyk}{0,0,0,0.1}
\definecolor{grey}{cmyk}{0,0,0,0.5}
\definecolor{darkgreen}{cmyk}{0.6,0,0.8,0}
\newcommand{\Red}[1]{{\color{red}  {#1}}}

\newcommand{\Green}[1]{{\color{darkgreen}  {#1}}}
\newcommand{\Blue}[1]{\color{blue}{#1}\color{black}}
\newcommand{\Brown}[1]{{\color{brown}{#1}\color{black}}}

\newenvironment{proofof}[2]{\par\vspace{2mm}\noindent\textbf{Proof of {#1} {#2}}\ }{\hfill\BlackBox\\[2mm]}

\ifx\proof\undefined
\newenvironment{proof}{\par\noindent{\bf Proof\ }}{\hfill\BlackBox\\[2mm]}
\fi

\ifx\theorem\undefined
\newtheorem{theorem}{Theorem}
\fi

\ifx\example\undefined
\newtheorem{example}{Example}
\fi

\ifx\property\undefined
\newtheorem{property}{Property}
\fi

\ifx\lemma\undefined
\newtheorem{lemma}{Lemma}
\fi

\ifx\proposition\undefined
\newtheorem{proposition}{Proposition}
\fi

\ifx\remark\undefined
\newtheorem{remark}{Remark}
\fi

\ifx\corollary\undefined
\newtheorem{corollary}{Corollary}
\fi

\ifx\definition\undefined
\newtheorem{definition}{Definition}
\fi

\ifx\conjecture\undefined
\newtheorem{conjecture}[theorem]{Conjecture}
\fi

\ifx\axiom\undefined
\newtheorem{axiom}[theorem]{Axiom}
\fi

\ifx\claim\undefined
\newtheorem{claim}[theorem]{Claim}
\fi

\ifx\assumption\undefined
\newtheorem{assumption}{Assumption}
\fi

\def\Bb{\overline{\B}}
\def\Csize{m}

\def\ie{i.e.,\xspace}
\def\eg{e.g.,\xspace}
\def\whp{w.h.p.,\xspace}
\def\nystrom{Nystr\"om\xspace}
\def\dppvfx{\textsc{DPP-VFX}\xspace}

\newcommand{\sizeS}[1]{|S_{#1}|}
\ifisneurips\newcommand{\mode}{M}\fi

\title{Exact sampling of determinantal point processes\\
  with sublinear time preprocessing}

\ifisarxiv\date{}\def\And{\and}\fi
\author{
      Micha{\l } Derezi\'{n}ski\thanks{Equal contribution.}\\
Department of Statistics\\
University of California, Berkeley\\
\texttt{mderezin@berkeley.edu}
\And
Daniele Calandriello\footnotemark[1]\\
LCSL\\
Istituto Italiano di Tecnologia, Italy\\
\texttt{daniele.calandriello@iit.it}
\And
Michal Valko\\
DeepMind Paris\\
\texttt{valkom@google.com}
}

\begin{document}

\maketitle

\begin{abstract}
We study the complexity of sampling from a distribution over all index
subsets of the set $\{1,...,n\}$ with the probability of a subset $S$
proportional to the determinant of the submatrix $\L_S$ of some $n\times
n$ p.s.d.\,matrix $\L$, where $\L_S$ corresponds to the entries of $\L$ indexed by 
$S$. Known as a determinantal point process, this distribution is 
used in machine learning to induce diversity in subset selection. In
practice, we often wish to sample \emph{multiple} subsets $S$ with small
expected size $k \triangleq \E[|S|] \ll n$ from a very large matrix $\L$, so it is
important to minimize the preprocessing cost of the procedure
(performed once) as well as the sampling cost (performed
repeatedly). For this purpose, we propose a new algorithm which, given
access to $\L$, samples \emph{exactly} from a determinantal point process while
satisfying the following two properties: (1) its preprocessing cost is
$n \cdot \poly(k)$, i.e., \emph{sublinear} in the size of~$\L$, and (2) its sampling cost is
$\poly(k)$, i.e., \emph{independent} of the size of~$\L$. Prior to our results,  
state-of-the-art exact samplers required $\cO(n^3)$ preprocessing time
and sampling time linear in $n$ or dependent on the spectral
properties of $\L$. We also give a
reduction which allows using our algorithm for exact sampling
from cardinality constrained determinantal point processes with
$n\cdot\poly(k)$ time preprocessing.
\end{abstract}


\section{Introduction}
Given a positive semi-definite (psd) $n\times n$ matrix $\L$,
a determinantal point process $\DPP(\L)$ (also known as an $L$-ensemble) is a distribution over all
$2^n$ index subsets $S\subseteq\{1,\dots, n\}$ such that
\begin{align*}
  \Pr(S) \triangleq \frac{\det(\L_S)}{\det(\I+\L)}\CommaBin
\end{align*}
where $\L_S$ denotes the $|S|\times |S|$ submatrix of $\L$ with rows
and columns indexed by $S$.
Determinantal point processes naturally appear across many scientific
domains \cite{dpp-physics,dpp-stats,spanning-trees}, while also 
being used as a tool in machine learning and recommender systems
\cite{dpp-ml} for inducing diversity in subset selection and as a
variance reduction approach. In this context, we often wish to
efficiently produce many DPP samples of small expected size\footnote{
To avoid complicating the exposition with edge cases, we assume $k \geq 1$. Note that this
can be always satisfied without distorting the distribution by rescaling $\Lb$
by a constant, and is without
loss of generality, as our analysis can be trivially extended to the case $0 \leq k < 1$ with
some additional notation.
}
$k\triangleq\E[|S|]$ given a large matrix~$\L$. Sometimes the
distribution is restricted to subsets of fixed size $|S|=k\ll n$, denoted $k$-$\DPP(\L)$.
\cite{dpp-independence} gave an
algorithm for drawing  samples from $\DPP(\L)$ distributed exactly, later
adapted to $k$-$\DPP(\L)$ by \cite{k-dpp}, which can
be implemented to run in polynomial time. In many applications,
however, sampling is  still a computational bottleneck because the algorithm
requires performing the eigendecomposition of 
matrix $\L$ at the cost of $\cO(n^3)$. In addition to that initial cost, producing
many independent samples $S_1,S_2,\dots$ at high frequency poses a challenge because
the cost of each sample is at least linear in $n$. Many 
alternative algorithms have been considered for both DPPs
and $k$-DPPs to reduce the computational cost of preprocessing
and/or sampling, including many approximate and heuristic 
approaches. Contrary to approximate solutions, we present an algorithm which samples 
\emph{exactly} from a DPP or a $k$-DPP with the
initial preprocessing cost \emph{sublinear} in the size of
$\L$ and the sampling cost \emph{independent} of the size of $\L$. 
\begin{theorem}\label{t:main}
For a psd $n\times n$ matrix $\L$, let $S_1,S_2,\dots$ be i.i.d.\,random
sets from $\DPP(\L)$ or from any $k$-$\DPP(\L)$.
Then, there is an algorithm which, given access to $\L$, returns
\begin{enumerate}[label=\alph*)]
\item the first subset, $S_1$, in:\quad $n\cdot\poly(k)\,\polylog(n)$ time,
\item each subsequent $S_i$ in:\hspace{6.4mm}  $\poly(k)$ time.
  \end{enumerate}
\end{theorem}
We refer to this algorithm as the very fast and exact DPP
sampler, or \dppvfx.
 Table \ref{tab:comp} compares \dppvfx with other
DPP and $k$-DPP sampling algorithms. In this comparison, we focus on the methods
that provide strong accuracy guarantees. As seen from the
table, our algorithm is the first exact sampler to 
achieve sublinear overall runtime. Only the approximate MCMC sampler
of \cite{rayleigh-mcmc} matches our $n\cdot\poly(k)$ complexity
(and only for a $k$-DPP), but for this method every next
sample is equally expensive, making it less practical when repeated
sampling is needed. In fact, to our knowledge, no other exact or
approximate method (with rigorous approximation guarantees)
achieves $\poly(k)$ sampling time of this paper. 

\begin{table}
  \centering
\begin{tabular}{r||c|c|c|l|l}
 &exact&DPP&$k$-DPP&first sample&subsequent samples\\
  \hline\hline
  \cite{dpp-independence,k-dpp}&\cmark&\cmark&\cmark&$n^3$&$nk^2$\\
  \cite{rayleigh-mcmc}&\xmark&\xmark&\cmark&$n\cdot\poly(k)$&$n\cdot\poly(k)$\\
  \cite{kdpp-mcmc}&\xmark&\cmark&\xmark&$n^2\cdot\poly(k)$
&$n^2\cdot\poly(k)$\\
  \cite{dpp-noeig}&\cmark&\cmark&\xmark&$n^3$&$\poly(k\cdot(1+\|\L\|) )$\\
  \cite{dpp-intermediate}&\cmark&\cmark&\xmark&$n^3$&$\poly(\rank(\L))$\\[1mm]
  \hline
\dppvfx \textbf{(this paper)}&\cmark&\cmark&\cmark&$n\cdot \poly(k)$&$\poly(k)$
\end{tabular}
\vspace{2mm}
\caption{Comparison of DPP and k-DPP algorithms using the L-ensemble
  representation. For a DPP, $k$ denotes the expected subset
  size. Note that $k\leq \rank(\L)\leq n$. We omit log terms for clarity.}
\label{tab:comp}
\vspace{-5mm}
\end{table}

Our method is based on a
technique developed recently by
\cite{leveraged-volume-sampling,correcting-bias} and later extended
by \cite{dpp-intermediate}. In this approach, we
carefully downsample the index set $[n]=\{1,...,n\}$ to a sample
$\sigma=(\sigma_1,...,\sigma_t)\in[n]^t$ that is small but still sufficiently
larger than the expected target size $k$, and then run a DPP on
$\sigma$. As the downsampling distribution we use a \emph{regularized}
determinantal point process (R-DPP), proposed by
\cite{dpp-intermediate}, which (informally) samples $\sigma$
with probability $\Pr(\sigma)\sim \det(\I+\Lbt_{\sigma})$, where
$\Lbt$ is a rescaled version of $\L$.
Overall, the approach is described in the diagram below, where
$|S|\leq t\ll n$,
\begin{align*}
\{1,...,n\}\ \ \xrightarrow{\sigma\sim \text{R-DPP}}\ \ (\sigma_1,...,\sigma_t)
\ \ \xrightarrow{\St\sim\text{DPP}}\ \ S =\{\sigma_i:i\in \St\}.
\end{align*}
The DPP algorithm proposed by \cite{dpp-intermediate} follows the
same diagram, however it requires that the size of the intermediate
sample $\sigma$ be $\Omega(\rank(\L)\cdot k)$. This means that their method
provides improvement over \cite{dpp-independence} only when $\L$ can be decomposed as $\X\X^\top$
for some $n\times r$ matrix $\X$, with $r\ll n$. However, in practice,
matrix $\L$ is often only \emph{approximately low-rank}, i.e., it exhibits some form
of eigenvalue decay but it does not have a low-rank factorization. In
this case, the results of \cite{dpp-intermediate} are vacuous both in
terms of the preprocessing cost and the sampling cost, in that
obtaining \emph{every} sample would take $\Omega(n^3)$. We propose a different
R-DPP implementation (see \dppvfx as  Algorithm \ref{alg:main}) where the expected
size of $\sigma$ is $\cO(k^2)$. To make the algorithm efficient, we use
new connections between determinantal point processes, ridge leverage
scores, and Nystr\"om approximation.
\begin{definition}\label{d:rls}
  Given a psd matrix $\L$, its $i$th $\lambda$-ridge leverage score (RLS)
  $\tau_i(\lambda)$ is the $i$th diagonal entry of
  $\L(\lambda\I+\L)^{-1}$. The $\lambda$-effective dimension $\deff(\lambda)$
  is the sum of the leverage scores, 
$\sum_{i}\tau_i(\lambda)$.
\end{definition}
An important connection between RLSs and DPPs is that when $S\sim \DPP(\L),$ the marginal probability of index $i$
being sampled into $S$ is equal to the $i$th $1$-ridge leverage score
of $\L$, and the expected size $k$ of $S$ is equal to the $1$-effective dimension:
\begin{align*}
  \Pr(i\in S) = \big[\L(\I+\L)^{-1}\big]_{ii} = \tau_i(1),\qquad
  k \triangleq \E\big[|S|\big] = \tr\big(\L(\I+\L)^{-1}\big) = \deff(1).
\end{align*}
Intuitively, if the marginal probability of $i$ is high, then this
index should  likely make it into the intermediate sample
$\sigma$. This suggests that i.i.d.\,sampling of the indices $\sigma_1,...,\sigma_t$ proportionally to
1-ridge leverage scores, i.e.~$\Pr(\sigma_1=i)\propto \tau_i(1)$, should serve as a reasonable and cheap
heuristic for constructing $\sigma$. In fact, we can show that this
distribution can be easily corrected by rejection sampling to
become the R-DPP that we need. Computing
ridge leverage scores exactly costs $\cO(n^3)$, so instead we compute
them approximately by first constructing a Nystr\"om approximation of $\L$. 
\begin{definition}
  Let $\L$ be a psd matrix and $C$ a subset of its row/column indices
  with size $\Csize  \triangleq |C|$.
Then we define the \nystrom approximation of $\L$ based on $C$ as the $n\times n$
  matrix $\Lbh \triangleq (\L_{C,\Ic})^\top\L_C^+\L_{C,\Ic}$. 
\end{definition}
Here, $\L_{C,\Ic}$ denotes an $\Csize \times n$ matrix consisting of (entire)
rows of $\L$ indexed by $C$ and $(\cdot)^+$ denotes the Moore-Penrose
pseudoinverse.
Since we use rejection sampling to achieve the
right intermediate distribution, the correctness of our algorithm does not depend on
which \nystrom approximation is chosen. However, the subset $C$ greatly influences
the computational cost of the sampling through the rank of $\Lbh$ and the
probability of rejecting a sample. Since $\rank(\Lbh) = |C|$, operations
such as multiplication and inversion involving the \nystrom approximation will
scale with $m$, and therefore a small subset increases efficiency.
However, if $\Lbh$ is too different from $\L$, the probability of rejecting the sample will be very high and the
algorithm inefficient. In this case, a slightly larger subset could
improve accuracy and acceptance rate without increasing too much
the cost of handling $\Lbh$.
Therefore, subset $C$ has to be selected so that it is both small and accurately
represents the matrix $\L$. Here, we once again rely on ridge
leverage score sampling which has been effectively used for obtaining good \nystrom
approximations  in a number of
prior works such as \cite{ridge-leverage-scores,calandriello_disqueak_2017,NIPS2018_7810}.

While our main algorithm can sample only from the \emph{random-size} 
DPP, and not from the \emph{fixed-size} $k$-DPP, we present a rigorous
reduction argument which lets us use our DPP algorithm to sample
exactly from a $k$-DPP (for any~$k$) with  a small computational overhead.


\paragraph{Related work}
\label{s:related}
Prior to our work, fast exact sampling from generic DPPs has been considered out of reach.
The first procedure to sample general DPPs was given by \cite{dpp-independence} 
and even most recent exact refinements
\cite{dpp-noeig,dpp-intermediate,poulson2019high-performance}, when
the DPP is represented in the form of an $L$-ensemble, require preprocessing that amounts to an expensive 
$n \times n$  matrix diagonalization at the cost $\cO(n^3)$, which is
shown as the \emph{first-sample} complexity column in
Table~\ref{tab:comp}. 

Nonetheless, there are well-known samplers for very specific DPPs
that are both fast and exact, for instance for sampling
uniform spanning trees \cite{Ald90,Bro89,PrWi98JoA}, 
which leaves the possibility of a more generic fast sampler open.
Since the sampling from DPPs
has several practical large scale machine learning applications
\cite{dpp-ml}, there are now
a number of methods known to be able to sample from a
DPP \emph{approximately}, outlined in the following paragraphs.

As DPPs can be specified by kernels ($L$-kernels or $K$-kernels), a natural 
approximation strategy is to resort to low-rank approximations \cite{k-dpp, dpp-salient-threads,dpp-nystrom,dpp-coreset}.
For example, \cite{dpp-nystrom}  provides approximate guarantee for the probability of any subset 
being sampled as a function of eigengaps of the $L$-kernel.
Next, \cite{dpp-coreset} construct \emph{coresets}
approximating a given $k$-DPP  
and then use them for sampling. In their Section~4.1, \cite{dpp-coreset} show in which cases we can 
hope for a good approximation. 
These guarantees become tight if these approximations (\nystrom subspace, coresets)
are aligned with data. In our work, we aim for an \emph{adaptive} approach that is able to provide a good
approximation for \emph{any} DPP. 
	
The second class of approaches are based on Markov chain Monte-Carlo
\cite{metropolis1949monte}
techniques
\cite{dpp-clustering,ReKa15,rayleigh-mcmc,kdpp-mcmc,gautier2017zonotope}. There
are known  polynomial bounds on the mixing rates  \cite{diaconis1991geometric} of MCMC chains with arbitrary DPPs as their limiting measure. In
particular, \cite{rayleigh-mcmc} 
showed them for cardinality-constrained DPPs and \cite{kdpp-mcmc} for
the general case. The two chains have mixing times which are, respectively,
linear and quadratic in $n$ (see Table~\ref{tab:comp}).
Unfortunately, for any subsequent sample we need to wait until the 
chain \emph{mixes again.}

Neither the known low-rank approximations or the known MCMC methods
are able to provide samples that are 
\emph{exactly} distributed (also called \emph{perfect 
sampling}) according to a DPP.
This is not surprising as having \emph{scalable and exact} sampling is very challenging in general.
For example, methods based on rejection 
sampling are \emph{always exact}, but they typically do not scale with the dimension 
and are adversely affected by the spikes in the distribution \cite{erraqabi2016pliable},
resulting in high rejection rate and inefficiency. 
Surprisingly, our method is based on both low-rank approximation
(a source of inaccuracy) and rejection sampling (a common source of inefficiency).
In the following section, we show how to obtain a \emph{perfect DPP
sampler} from a \nystrom approximation of the $L$-kernel.
Then, to guarantee efficiency, in Section~\ref{s:fast} we bound the number of
rejections, which is possible thanks to the
use of intermediate downsampling.
\section{Exact sampling using any Nystr\"om approximation}
\label{s:exact}
\begin{wrapfigure}{r}{0.5\textwidth}
  \vspace{-8mm}
  \begin{minipage}{0.5\textwidth}
\begin{algorithm}[H]
  \caption{\ifisarxiv\small\dppvfx sampling $S\!\sim\!\DPP(\Lb)$
  \else\dppvfx~sampling $S\sim\DPP(\Lb)$\fi}
\label{alg:main}
\vspace{1mm}
{\footnotesize \textbf{input:}}\\[\ifisarxiv-7mm\else-8mm\fi]
  \begin{center}
\parbox{\ifisarxiv 0.85\else0.74\fi\textwidth}{
    \scriptsize
$\L\in\R^{n\times n}$, its Nystr\"om
approximation $\Lbh$, 
\ $q>0$\\[1mm]
$l_i=\big[(\L - \Lbh) +\Lbh(\I+\Lbh)^{-1}\big]_{ii}\ \approx\
\Pr(i\in S)$,  \\
$\sh =  \sum_i l_i$, \quad $\st = \tr\big(\Lbh(\I+\Lbh)^{-1}\big)$,\\
$\Lbt = 
\frac\sh q\Big[\frac1{\sqrt{l_il_j}}\,L_{ij}\Big]_{ij}$
}
\end{center}
\begin{algorithmic}[1]
  \STATE \textbf{repeat}\label{line:rep1}
  \vspace{1mm}
  \STATE \quad sample $t \sim
  \mathrm{Poisson}(q\,\ee^{\sh/q})$
  \vspace{1mm}
  \STATE \quad sample $\sigma_1,\tinydots,\sigma_t\simiid (\frac{l_1}{\sh}\CommaBin\tinycdots\CommaBin\frac{l_n}{\sh})$,
  \label{line:iid}
  \vspace{1mm}
   \STATE \quad sample $\textit{Acc}\sim\!
  \text{Bernoulli}\Big(\frac{\ee^{\st}\det(\I+\Lbt_\sigma)}  
{\ee^{t\sh/q}\det(\I+\Lbh)}\Big)$ \label{line:acc}
\STATE \textbf{until} $\textit{Acc}=\text{true}$\label{line:rep2}
\vspace{1mm}
\STATE sample $\St\sim \DPP\big(\Lbt_\sigma\big)$\label{line:sub}
  \RETURN $S = \{\sigma_i: i\!\in\! \St\}$
\end{algorithmic}
\end{algorithm}
\end{minipage}
\vspace{-4mm}
\end{wrapfigure}
\paragraph{Notation}
We use $[n]$ to denote the set $\{1,\dots,n\}$. For a matrix
$\B\in\R^{n\times m}$ and index sets $C$, $D$, we use $\B_{C,D}$
to denote the submatrix of $\B$ consisting of the intersection of rows
indexed by $C$ with columns indexed by $D$.  If $C=D$, we use a
shorthand $\B_{C}$ and if $D=[m]$, we may write $\B_{C,\Ic}$.
Finally, we also allow $C,D$ to be multisets (or
sequences), in which case each entry is rescaled by the multiplicities of its
row and column indices as follows: the entry at the intersection of
the $i$th row with the $j$th column
becomes $\sqrt{n_im_j}\cdot B_{ij}$, where $n_i$ is the multiplicity
of $i$ in $C$ and $m_j$ is the multiplicity of $j$ in $D$. Such a
rescaling has the property that if $\L=\B\B^\top$ then
$\L_{C,D}=\B_{C,\Ic}\B_{D,\Ic}^\top$.

As discussed in the introduction, our method relies on an intermediate
downsampling distribution to reduce the size of the problem. The
exactness of our sampler relies on the careful choice of that
intermediate distribution. To that end, we use regularized
determinantal processes, introduced by \cite{dpp-intermediate}. In the
below definition, we adapt them to the kernel setting.
\begin{definition}\label{d:r-dpp}
  Given an $n\times n$ psd matrix $\L$, distribution
  $p\triangleq(p_1,\dots,p_n)$ and $r>0$, let $\Lbt$ denote an $n\times n$
  matrix such that $\Lt_{ij} \triangleq \frac1{r\sqrt{p_ip_j}}L_{ij}$ for all
  $i,j\in[n]$. We define $\Vol_p^r(\L)$ as a
  distribution over events $A\subseteq\bigcup_{k=0}^\infty[n]^k$, where
  \begin{align*}
    \Pr(A) =\frac{\E_{\sigma}\big[\one_{[\sigma\in A]}\det(\I +
    \Lbt_\sigma)\big]}
    {\det(\I + \L)}\CommaBin\quad\text{for}
    \quad\sigma=(\sigma_1,\dots,\sigma_t)\simiid p,\quad
    t\sim\mathrm{Poisson}(r).    
  \end{align*}
\end{definition}
Since the term $\det(\I+\Lbt_\sigma)$ has the same form as
the normalization constant of $\DPP(\Lbt_\sigma)$, an easy
calculation shows that the R-DPP can be used as an intermediate
distribution in our algorithm without introducing any distortion in the sampling.
\begin{proposition}[{\citealp[Theorem 8]{dpp-intermediate}}]\label{t:composition}
  For any $\L$, $p$, $r$ and $\Lbt$ defined as in Definition
  \ref{d:r-dpp},
  \begin{align*}
    \text{if}\quad \sigma\sim\Vol_p^r(\L)\quad\text{and}\quad
    S\sim\DPP(\Lbt_\sigma)\quad\text{then}\quad
    \{\sigma_i:i\!\in\! S\}\sim\DPP(\L).
  \end{align*}
\end{proposition}
To sample from the R-DPP, \dppvfx uses rejection
sampling, where the proposal distribution is sampling i.i.d.\,proportionally to the approximate 1-ridge leverage scores $l_i\approx
\tau_i(1)$ (see Definition~\ref{d:rls} and the following discussion), computed using any
\nystrom approximation $\Lbh$ of matrix $\L$. Apart from $\Lbh$, the
algorithm also requires an additional parameter $q$, which controls
the size of the intermediate sample. Because of rejection sampling and
\Cref{t:composition}, the correctness of the algorithm \emph{does not
depend} on the choice of $\Lbh$ and $q$, as demonstrated in the following
result. The key part of the proof involves showing that the acceptance
probability in Line~\ref{line:acc} is bounded by 1. Here, we obtain 
a considerably tighter bound than the one achieved by
\cite{dpp-intermediate}, which allows us to use a much smaller intermediate
sample $\sigma$ (see Section \ref{s:fast}) while maintaning the efficiency of rejection sampling.
\begin{theorem}\label{t:exact}
Given a psd matrix $\L$, any one of its Nystr\"om approximations
$\Lbh$ and any positive $q$, \dppvfx
  returns $S\sim\DPP(\L)$.
\end{theorem}
\begin{proof}
We start by showing that the Bernoulli probability in Line~\ref{line:acc} is bounded by 1. Note that this is important not only
to sample correctly, but also when we later establish the efficiency
of the algorithm. If we showed a weaker upper bound, say $c>1$, we
could always divide the expression by $c$ and retain the correctness,
however it would also be $c$ times less likely that $\textit{Acc}=1$.

Since $\Lbh$ is a Nystr\"om
approximation for some $C\subseteq \Ic$, it can be written as 
\begin{align*}
\Lbh = \L_{\Ic,C}\L_C^+\L_{C,\Ic} =
  \B\B_{C,\Ic}^\top\L_C^+\B_{C,\Ic}\B^\top
  = \B\P\B^\top,
\end{align*}
for any $\B$ such that $\L=\B\B^\top$, where $\P\triangleq\B_{C,\Ic}^\top\L_C^+\B_{C,\Ic}$ is a
projection (so that $\P^2=\P$).
Let $\Lbt\triangleq\Bt\Bt^\top$, where the $i$th row of $\Bt$ is the rescaled
$i$th row of $\B$, i.e.~$\bt_i^\top\triangleq\sqrt{\!\frac\sh{q l_i}}\,\b_i^\top$. 
Then, we have\vspace{-3mm}
\begin{align*}
  \frac{\det (\I+\Lbt_\sigma)}{\det(\I+\Lbh)}
  &=\frac{\det (\I+\Bt_{\sigma,\Ic}^\top\Bt_{\sigma,\Ic})}{\det(\I+\P\B^\top\B\P)}
 =
    \det\!\big((\I+\Bt_{\sigma,\Ic}^\top\Bt_{\sigma,\Ic})(\I+\P\B^\top\B\P)^{-1}\big)
\\ &=
    \det\!\big(\I + (\Bt_{\sigma,\Ic}^\top\Bt_{\sigma,\Ic}-\P\B^\top\B\P)(\I+\P\B^\top\B\P)^{-1}\big)\\
  &\leq
    \exp\!\Big(\tr\big((\Bt_{\sigma,\Ic}^\top\Bt_{\sigma,\Ic}-\P\B^\top\B\P)(\I+\P\B^\top\B\P)^{-1}\big)\Big)
  \\ &=\exp\!\bigg(\sum_{i=1}^t \frac\sh{q l_{\sigma_i}}
       \big[\B(\I+\P\B^\top\B\P)^{-1}\B^\top\big]_{\sigma_i\sigma_i}\bigg)\cdot\mathrm{e}^{-\st}
\ =\  \ee^{t\sh/q}\ee^{-\st},
\end{align*}
where the last equality follows because
\begin{align*}
  \B(\I+\P\B^\top\B\P)^{-1}\B^\top
  &= \B\big(\I - \P\B^\top(\I+\B\P\B^\top)^{-1}\B\P\big)\B^\top
  \\ &= \L - \Lbh(\I+\Lbh)^{-1}\Lbh
= \L - \Lbh + \Lbh(\I+\Lbh)^{-1}.
\end{align*}
Thus, we showed that the expression in Line~\ref{line:acc} is
valid. Let $\sigmat$ denote the random variable distributed as
$\sigma$ is after exiting the repeat loop. It follows that
\begin{align*}
  \Pr(\sigmat\in A)
  &\propto \E_{\sigma}
  \bigg[\one_{[\sigma\in A]}\frac{\ee^{\st}\det(\I+\Lbt_\sigma)}{\ee^{t\sh/q}\det(\I+\Lbh)}\bigg]
\propto \sum_{t=0}^\infty
\frac{(q\,\ee^{\sh/q})^t}{\ee^{q\,\ee^{\sh/q}}\,t!}\cdot
       \ee^{-t\sh/q}\,
       \E_\sigma\big[\one_{[\sigma\in A]}\det(\I+\Lbt_\sigma)\mid t\big]
\\ &\propto \E_{t'}\Big[\E_\sigma\big[\one_{[\sigma\in A]}\det(\I+\Lbt_\sigma)\mid
       t=t'\big]\Big]\quad\text{for } t'\sim\mathrm{Poisson}(q),
\end{align*}
which shows that $\sigmat\sim \Vol_l^{q}(\L)$ for
$l=(\frac{l_1}\sh\!\CommaBin \tinycdots\CommaBin\!\frac{l_n}\sh)$. The claim follows from \Cref{t:composition}.
\end{proof}
Even though the choice of $\Lbh$ affects the overall execution of
the algorithm, it does not affect the output so we can reuse the same
$\Lbh$ to produce multiple independent samples
$S_1,S_2,...\sim\DPP(\L)$ (proof in Appendix \ref{a:proofs}).
\begin{lemma}\label{l:independent}
Let $C\subseteq [n]$ be a random set variable with any
distribution. Suppose that  $S_1$ and $S_2$ are returned by two executions of \dppvfx, both using inputs constructed from the same $\L$
  and $\Lbh=\L_{\Ic,C}\L_C^+\L_{C,\Ic}$. Then $S_1$ and $S_2$ are
  (unconditionally) independent. 
\end{lemma}

\section{Conditions for fast sampling}\label{s:fast}
The complexity cost of \dppvfx can be roughly summarized as follows:
we pay a large one-time cost to precompute $\Lbh$ and all its associated quantities,
and then we pay a smaller cost in the rejection sampling scheme
which must be multiplied by the number of times we repeat the loop until
acceptance.
We first show that if the sum of the approximate RLS (\ie $\sum_i l_i$, denoted
by $\sh$) is sufficiently close to $k$, then we will exit the loop
with high probability. We then analyze how accurate the precomputing
step needs to be to satisfy this condition.

\textbf{Bounding the number of rejections.} \ 
The following result presents the two conditions needed for achieving
efficient rejection sampling in \dppvfx. First, the \nystrom
approximation needs to be accurate enough, and second, the intermediate
sample size (controlled by parameter $q$) needs to be
$\Omega(k^2)$. This is a significant improvement over the guarantee of
\cite{dpp-intermediate} where the intermediate sample size is
$\Omega(\rank(\L)\cdot k)$, which is only meaningful for low-rank
kernels. The main novelty in this proof comes from showing the following
lower bound on the ratio of the determinants of $\I+\L$ and $\I+\Lbh$,
when $\Lbh$ is a \nystrom approximation:
$\det(\I+\L)/\det(\I+\Lbh)\geq\ee^{k-\st}$, where
$\st\triangleq\tr(\Lbh(\I+\Lbh)^{-1})$. Remarkably, this bound exploits 
the fact that any \nystrom approximation of $\L=\B\B^\top$ can be written as
$\Lbh=\B\P\B^\top$, where $\P$ is a projection matrix.
Note that while our result holds in the worst-case, in practice
the conditions on $\Lbh$ and on $q$ can be considerably relaxed.
\begin{theorem}\label{thm:fast-sampling-result}
If the Nystr\"om approximation $\Lbh$ and the intermediate sample size
parameter $q$ satisfy
  \begin{align*}
    \tr\big(\Lb(\I+\Lb)^{-1}\Lb-\Lbh(\I+\Lb)^{-1}\Lbh\big)  \ \leq\ 1
\qquad\text{and}\qquad q\ \geq\ \max\{\sh^2,\sh\},
  \end{align*}
then $\Pr(\textit{Acc} = \text{true}) \geq e^{-2}$.
  Therefore, with probability $1-\delta$
\Cref{alg:main} exits the rejection sampling loop after at most
$O(\log\delta^{-1})$ iterations and, after precomputing all of the inputs,
the time complexity of the rejection sampling loop is
$O\big(k^6\log\delta^{-1} + \log^4\!\delta^{-1}\big)$.
\end{theorem}
\begin{proof}
Le $\sigma$ be distributed as in Line~\ref{line:iid}. The probability
of exiting the repeat loop at each iteration is
\begin{align*}
P \triangleq\E_{\sigma}\!
  \bigg[\frac{\ee^{\st}\det(\I+\Lbt_\sigma)}{\ee^{t\sh/q}\det(\I+\Lbh)}\bigg]\!
&=\sum_{t=0}^\infty \frac{(q\,\ee^{\sh/q})^t}{\ee^{q\,\ee^{\sh/q}}\,t!}\cdot
\frac{\ee^{\st-t\sh/q}}{\det(\I+\Lbh)}
\E_\sigma\!\big[\det(\I+\Lbt_\sigma)\mid t\big]
\\ &=\frac{\ee^{q-q\,\ee^{\sh/q} + \st}}{\det(\I+\Lbh)}
\sum_{t=0}^\infty\frac{q^t}{\ee^{q}t!}
     \E_\sigma\!\big[\!\det(\I+\Lbt_\sigma)\mid t\big] 
=\ee^{q-q\,\ee^{\sh/q} + \st}\,\frac{\det(\I+\L) }{\det(\I+\Lbh)}\CommaBin
\end{align*}
where the last equality follows because the infinite series
computes the normalization constant of $\Vol_l^q(\L)$ given in
Definition \ref{d:r-dpp}. If $\sh\geq 1$ then $q=\sh^2$ and the
inequality $\mathrm{e}^x\leq 1+x+x^2$ for $x\in[0,1]$ implies that
$q-q\,\ee^{\sh/q}+\st=\sh^2(1-\ee^{1/\sh}+1/\sh) + \st-\sh\geq
-1+\st-\sh$. On the other hand, if $q=\sh\in[0,1]$, then
$q-q\ee^{\sh/q}+\st = (1-\ee)\sh + \st\geq -1+\st-\sh$. We proceed to
lower bound the determinantal ratio. Here, let $\L=\B\B^\top$ and
$\Lbh=\B\P\B^\top$, where $\P$ is a projection matrix. Then,
\begin{align*}
  \frac{\det(\I+\L) }{\det(\I+\Lbh)}
  &=\frac{\det(\I+\B^\top\B) } {\det(\I+\P\B^\top\B\P)}
=    \det\!\big(\I -(\B^\top\B-\P\B^\top\B\P)
     (\I+\B^\top\B)^{-1}\big)^{-1}
  \\ &
\geq\exp\!\Big(\tr\big((\B^\top\B-\P\B^\top\B\P)
(\I+\B^\top\B)^{-1}\big)\Big)
\\ &=\exp\!\Big(\tr\big(\B(\I+\B^\top\B)^{-1}\B^\top\big) -
     \tr\big(\B\P(\I+\B^\top\B)^{-1}\P\B^\top\big)\Big)
\\ &=\exp\!\Big(k -
     \tr\big(\B\P(\I+\B^\top\B - \B^\top\B)(\I+\B^\top\B)^{-1}\P\B^\top\big)\Big)
\\ &=\exp\!\Big(k -
     \tr(\Lbh) + \tr\big(\B\P\B^\top\B(\I+\B^\top\B)^{-1}\P\B^\top\big)\Big)
\\ &=\exp\!\Big(k -
     \tr(\Lbh) + \tr\big(\B\P\B^\top(\I+\B\B^\top)^{-1}\B\P\B^\top\big)\Big)
\\ &=\exp\!\Big(k -
     \tr(\Lbh) + \tr\big(\Lbh(\I+\Lb)^{-1}\Lbh\big)\Big).
\end{align*}
Putting all together,
\begin{align*}
\ee^{q-q\,\ee^{\sh/q} + \st}
  \frac{\det(\I+\L) }{\det(\I+\Lbh)}
\geq\exp\!\Big(
-1 + z - s + k - \tr(\Lbh) + \tr\big(\Lbh(\I+\L)^{-1}\Lbh\big)\Big),
\end{align*}
and 
using the definitions $k=\tr\big(\L(\I+\L)^{-1}\big)$,
$s = \tr(\Lb - \Lbh + \Lbh(\Lbh + \I)^{-1})$,
and $z = \tr(\Lbh(\Lbh+\I)^{-1})$ on the middle term
$z- s + k - \tr(\Lbh)$ we have
\begin{align*}
\tr&\big(\Lbh(\Lbh + \I)^{-1} + \Lb(\I+\Lb)^{-1}
-\Lb + \Lbh - \Lbh(\Lbh + \I)^{-1}
- \Lbh\big)\\
&= \tr\big(\Lb(\I+\Lb)^{-1} -\Lb\big)
                = \tr\big(\Lb(\I+\Lb)^{-1} -\Lb(\I+\Lb)^{-1}(\I+\Lb)\big)
= -\tr\big(\Lb(\I+\Lb)^{-1}\Lb\big),
\end{align*}
and therefore
\begin{align*}
\ee^{q-q\,\ee^{\sh/q} + \st}
  \frac{\det(\I+\L) }{\det(\I+\Lbh)}
\geq\exp\!\Big(
-1 + \tr\big(\Lbh(\I+\Lb)^{-1}\Lbh\big) - \tr\big(\Lb(\I+\Lb)^{-1}\Lb\big)\Big),
    \end{align*}
so we obtain our condition.
Thus, with probability $1-\delta$ the main loop will be repeated
$\cO(\log\delta^{-1})$ times. Since the number of samples $t_i$ drawn from
$l$ in the $i$th iteration of the loop is a Poisson distributed random variable, a standard
Poisson tail bound implies that with probability $1-\delta$ all
iterations will satisfy $t_i=\cO(k^2+\log\delta^{-1})$. Since the dominant cost is computing the
determinant of the matrix $\I+\Lbt_\sigma$ in time $\cO(t_i^3)$, and the final step,
$\DPP(\Lbt_\sigma)$, requires its eigendecomposition (with the same
time complexity), 
the result follows.
\end{proof}
\textbf{Bounding the precompute cost.} \ 
All that is left now is to control the cost of the precomputation phase.
We will separate the analysis into two steps: how much it costs to choose $\Lbh$ to satisfy
the assumption of \Cref{thm:fast-sampling-result}, and how much it costs to compute everything
else given $\Lbh$ (see \Cref{a:proofs}).
\begin{lemma}\label{lem:control-s-k}
Let $\Lbh$ be constructed by sampling $m = \cO(k^3\log\frac n\delta)$
columns proportionally to their RLS. Then, with probability $1 - \delta$,
$\Lbh$ satisfies the assumption of \Cref{thm:fast-sampling-result}.
\end{lemma}
\vspace{-.5\baselineskip}
There exist many algorithms to sample columns proportionally to their
RLS. For example, we can take the \textsc{BLESS}
algorithm from \cite{NIPS2018_7810} with  the following guarantee.
\vspace{-.5\baselineskip}
\begin{proposition}[{\citealp[Theorem~1]{NIPS2018_7810}}]
There exists an algorithm that with probability $1-\delta$ samples~$m$ columns proportionally to their RLS
in  $\cO(nk^2\log^2\!\frac n\delta + k^3\log^4\!\frac n\delta + m)$ time.
\end{proposition}
\vspace{-.5\baselineskip}
We can now compute the remaining preprocessing costs, given a \nystrom
approximation $\Lbh$.
\vspace{-.5\baselineskip}
\begin{lemma}\label{l:computing}
Given $\Lbh$ with rank $\Csize$, we can compute $l_i$, $s$, $z$, and $\Lbt$ in $\cO(n\Csize^2 + \Csize^3)$ time.
\end{lemma}
We are finally ready to combine these results to fully characterize the computational
cost.
\vspace{-.5\baselineskip}
\newtheorem{repeatmainthm}{Theorem}
\begin{repeatmainthm}[restated for DPPs only]\label{thm:main}
For a psd $n\times n$ matrix $\L$,
 let $S_1,S_2$ be i.i.d.\,random
sets from $\DPP(\L)$.
Denote with $\Lbh$ a \nystrom
approximation of $\L$ obtained by sampling $m = \cO(k^3\log(n/\delta))$
of its columns proportionally to their RLS.
Then, with probability $1-\delta$, \dppvfx returns
\begin{enumerate}[label=\alph*)]
\item subset $S_1$ in:\quad $\cO(nk^6\log^2\!\frac n\delta + k^9\log^3\!\frac n\delta
  + k^3\log^4\!\frac n\delta)$ time,
\item then, $S_2$ in:\hspace{4.2mm} $\cO\big(k^6\log\frac 1\delta + \log^4\!\frac 1\delta\big)$ time.
\end{enumerate}
\end{repeatmainthm}
\textbf{Discussion.} \ Due to the nature of rejection sampling, as long as we exit the loop,
\ie we accept the sample, the output of
\dppvfx is guaranteed to follow the DPP distribution for any value of $m$.
In Theorem \ref{thm:main} we set $m = (k^3\log\frac n\delta)$
to satisfy \Cref{thm:fast-sampling-result} and guarantee
a constant acceptance probability in the rejection sampling loop,
but this might not be necessary or even desirable in practice.
Experimentally, much smaller values of $m$, starting
from $m = \Omega(k\log\frac n\delta)$ seem to be sufficient
to accept the sample, while at the same time a smaller $m$ greatly reduces
the preprocessing costs.
In general, we recommend to separate \dppvfx in three phases.
First, compute an accurate estimate of the RLS using off-the-shelf algorithms
in $\cO(nk^2\log^2\!\frac n\delta + k^3\log^4\!\frac n\delta)$ time. Then, sample
a small number $m$ of columns to construct an explorative
$\Lbh$, and try to run \dppvfx If the rejection sampling loop does
not terminate sufficiently fast, then we can reuse the RLS
estimates to compute a more accurate $\Lbh$ for a larger $m$.
Using a simple doubling schedule
for $m$, this procedure will quickly reach a regime where
\dppvfx is guaranteed to accept \whp maintaining its asymptotic
complexity, while at the same time resulting in faster
sampling in practice.

\section{Reduction from DPPs to k-DPPs}\label{s:k-dpp}
We next show that with a simple extra rejection sampling
step we can efficiently transform \emph{any} exact
DPP sampler into an exact $k$-DPP sampler.

A common heuristic to sample $S$ from a $k$-DPP is to
first sample $S$ from a DPP, and then reject
the sample if the size of $S$ is not exactly $k$.
The success probability of this procedure can be improved by appropriately rescaling
$\Lb$ by a constant factor $\alpha$,
\begin{align*}
  \text{sample}\quad S_{\alpha}\sim \DPP(\alpha\L),\quad\text{accept
  if }|S_{\alpha}|=k. 
\end{align*}
Note that rescaling the DPP by a constant $\alpha$ only changes the expected size
of the set $S_{\alpha}$, and not its distribution. Therefore,
if we accept only sets with size $k$, we will be sampling exactly from
our $k$-DPP. Moreover, if $k_{\alpha} = \E[|S_\alpha|]$
is close to $k$, the success probability will improve.
With a slight abuse of notation, in the context of $k$-DPPs
we will indicate with $k$ the desired size of $S_{\alpha}$,
\ie the final output size at acceptance,
and with $k_{\alpha} = \E[|S_{\alpha}|]$ the expected size of the
scaled DPP.

While this rejection sampling heuristic is widespread in the literature,
until now there has been no proof that the process
can be provably made efficient.
We solve this open question with two new results. First we show that
for an appropriate rescaling $\alpha^\star$ we only reject $S_{\alpha^\star}$
roughly $O(\sqrt{k})$ times. Then, we show how we can find
such an $\alpha^\star$ with a $\tcO(n\cdot\poly(k))$ time preprocessing step.

\begin{theorem}\label{lem:k-dpp-repeat}
There exists constant $C>0$ such that for any rank $n$ psd matrix $\L$ and
$k\in[n]$, there is $\alpha^\star > 0$ with the following property: if we
sample $S_{\alpha^\star}\sim\DPP(\alpha^\star \L)$, then
   $ \Pr(|S_{\alpha^\star}|=k)\geq \frac1{C\sqrt{k}}$.
 \end{theorem}
 The proof (see \Cref{a:k-dpp}) relies on a known Chernoff bound for
 the sample size $|S_{\alpha}|$ of a DPP. When applied na\"ively, the
inequality does not offer a lower bound on the probability of any single
 sample size, however we can show that the probability mass is
 concentrated on $O(\sqrt{k_{\alpha}})$ sizes. This leads to a
 lower bound on the sample size with the largest probability, \ie the mode of the
 distribution. Then, it remains to observe 
 that for any $k\in[n]$ we can always find $\alpha^\star$ for which
 $k$ is the mode of $|S_{\alpha^{\star}}|$.
We conclude that given $\alpha^\star$, the rejection sampling scheme
described above transforms  any $\poly(k)$
time DPP sampler into a $\poly(k)$ time $k$-DPP sampler.
It remains to find $\alpha^\star$ efficiently, which once again relies
on using a \nystrom approximation of $\L$.
\begin{lemma}
If $k \geq 1$, then there is an algorithm that finds $\alpha^\star$ in
$\tcO(n\cdot\poly(k))$ time.
\end{lemma}


\noindent\begin{minipage}{0.48\textwidth}
  \centering
  \includegraphics[height=4.7cm]{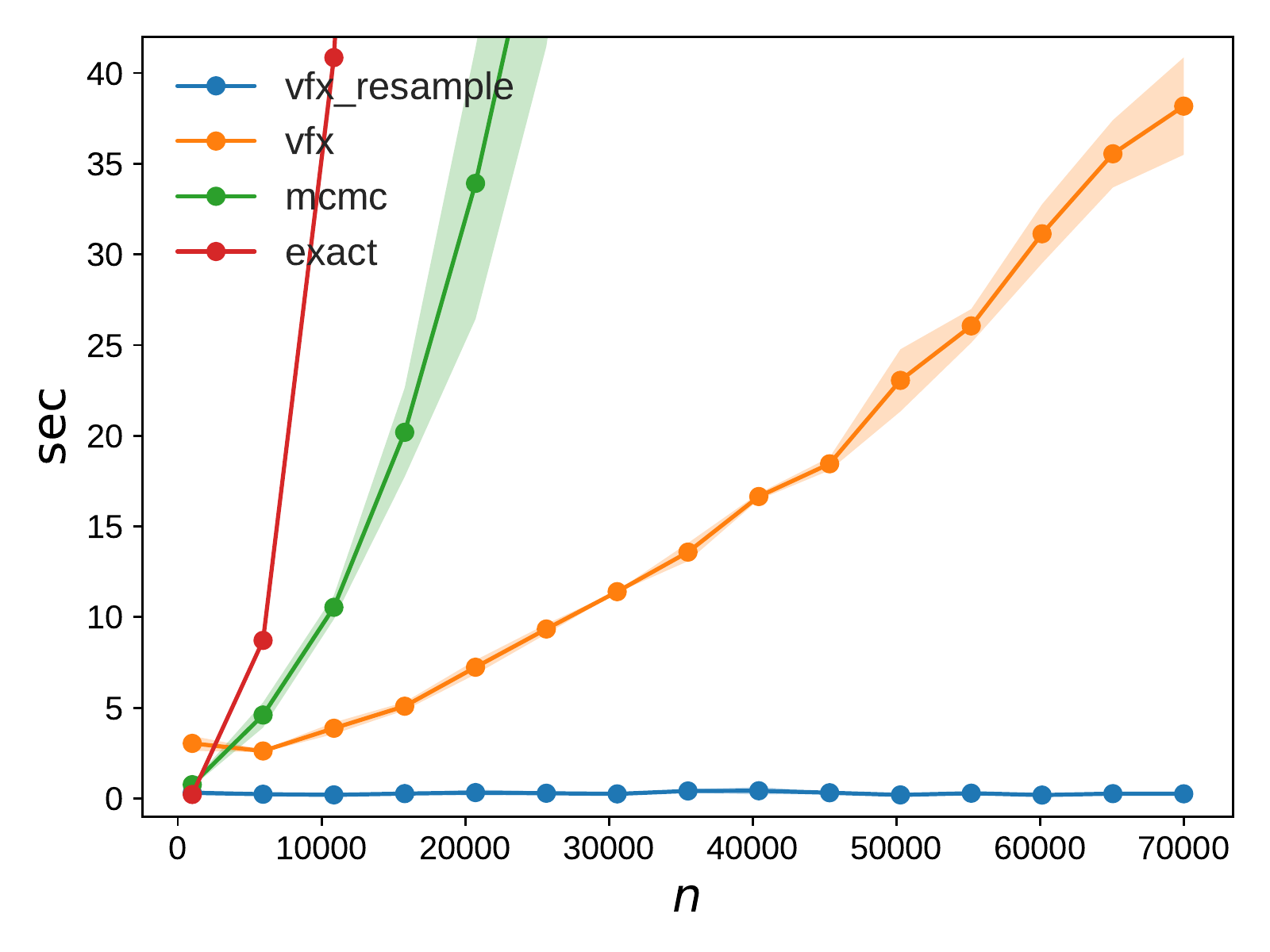}
  \vspace{-3mm}
  \captionof{figure}{\emph{First-sample} cost for \dppvfx against other
    DPP samplers ($n$ is data size).}\label{fig:exp}
\end{minipage}\hfill
\begin{minipage}{0.48\textwidth}
  \centering
  \includegraphics[height=4.7cm]{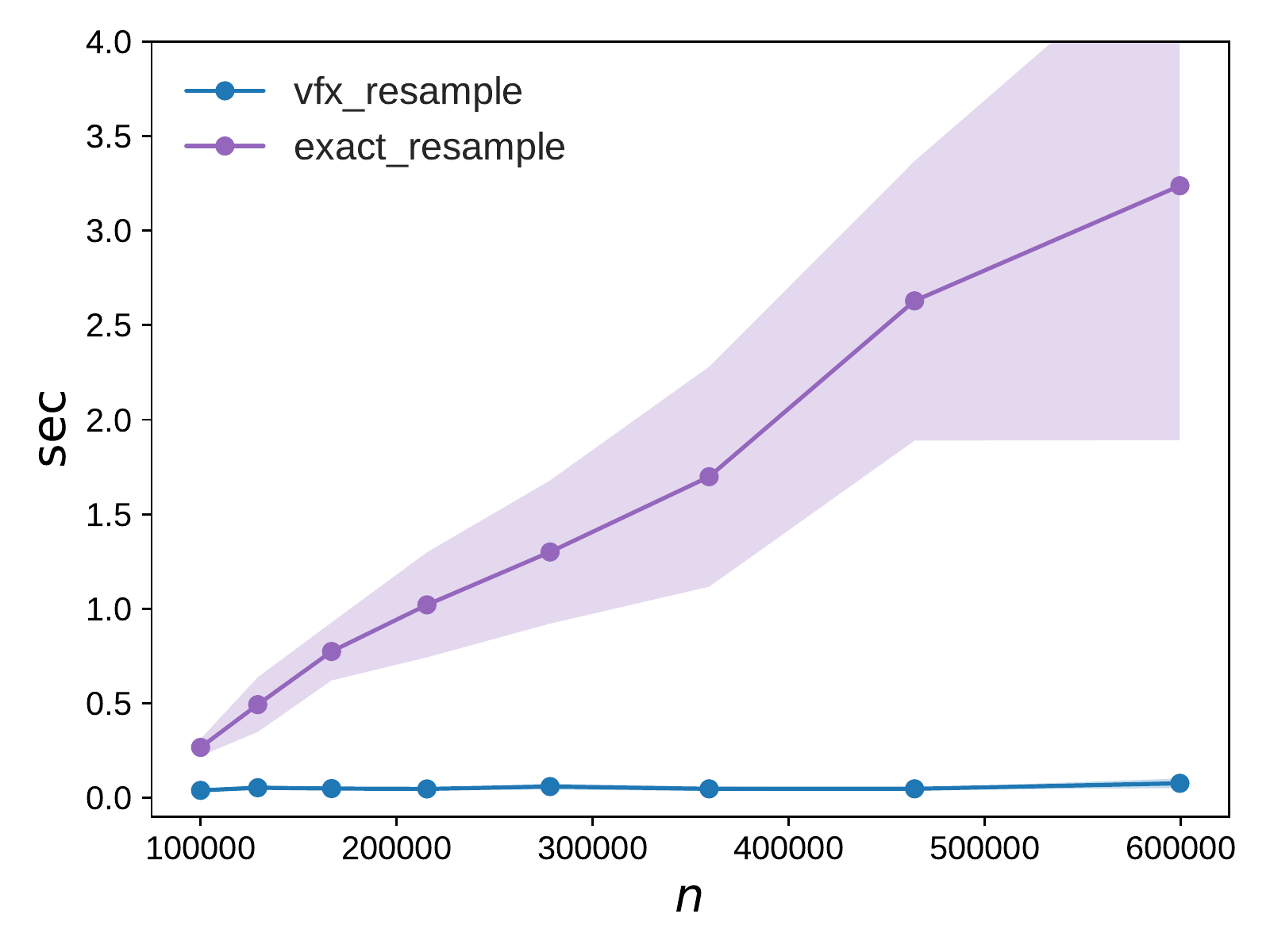}
  \vspace{-3mm}
  \captionof{figure}{\emph{Resampling} cost for \dppvfx compared to the exact
  sampler of \cite{dpp-independence}.}\label{fig:resample}
\end{minipage}
\section{Experiments}
In this section, we experimentally evaluate the performance of \dppvfx compared to exact sampling \cite{dpp-independence} and MCMC-based approaches \cite{rayleigh-mcmc}.
In particular, since \Cref{s:exact} proves that \dppvfx
samples exactly from the DPP, we are interested in evaluating
computational performance. This will be characterized by showing
how \dppvfx and baselines scale with the size $n$ of the matrix $\Lb$
when taking a first sample, and how \dppvfx achieves constant time
when resampling.

To construct $\Lb,$ we use random subsets of the infinite MNIST digits dataset
\cite{loosli-canu-bottou-2006},
where $n$ varies up to $10^6$ and $d = 784$. We use
an RBF kernel with $\sigma = \sqrt{3d}$ to construct $\Lb$.
All algorithms are implemented in \path{python}. For exact and MCMC sampling we used the \path{DPPy} library, \cite{gautier2018dppy}, while for \dppvfx
we reimplemented BLESS \cite{NIPS2018_7810}, and used \path{DPPy} to
perform exact sampling on the intermediate subset. All experiments
are carried out on a 24-core CPU and fully take advantage of potential parallelization.
For the \nystrom approximation
we set $m = 10\deff(1) \approx 10k$. While this is much lower than
the $\cO(k^3)$ value suggested by the theory, as we will see it is already accurate enough to result in drastic runtime improvements over exact and MCMC.
For each algorithm we control\footnote{For simplicity we do not perform the full $k$-DPP rejection step, but only adjust the expected size of the set.} the size of the output set by rescaling
the input matrix $\Lb$ by a constant, following the strategy
of \Cref{s:k-dpp}.
In \Cref{fig:exp} we report our results, means and 95\% confidence interval over 10 runs, for subsets of MNIST that go from $n =
10^3$ to $n=7\cdot10^4$, \ie the whole original MNIST dataset.

Exact sampling is clearly cubic in $n$, and we cannot push our
sampling beyond $n = 1.5\cdot10^4$. For MCMC, we enforce mixing by runnning the chain
for $nk$ steps, the minimum recommended by
\cite{rayleigh-mcmc}. However, for $n=7\cdot10^5$ 
the MCMC runtime is $358$ seconds and cannot be included in the plot,
while \dppvfx completes in $38$ seconds, an order of magnitude faster.
Moreover, \dppvfx rarely rejects more than 10 times, and the mode of
the rejections up to $n=7\cdot10^5$ is $1$, that is we mostly accept at the
first iteration.
\Cref{fig:resample} reports the cost of the second sample, \ie of resampling. For exact
sampling, this means that an eigendecomposition of $\Lb$ is already available,
but as the plot shows the resampling process still scales with $n$.
On the other hand, \dppvfx's complexity (after preprocessing) scales only with $k$
and remains constant regardless of $n$.

Finally, we scaled \dppvfx to $n = 10^6$ points, a regime where
neither exact nor MCMC approaches are feasible. We report runtime and
average rejections, with mean and 95\% confidence interval over 5
runs. \dppvfx draws its first sample
in $68.479 \pm 2.63$ seconds, with only $9\pm4.5$ rejections.



\subsubsection*{Acknowledgements}
MD thanks the NSF for funding via the NSF TRIPODS program.

\bibliographystyle{alpha}
\bibliography{pap}

\newpage
\appendix
\section{Omitted proofs for the main algorithm}\label{a:proofs}
In this section we present the proofs omitted from Sections \ref{s:exact} and \ref{s:fast}, which regarded the correctness and efficiency of \dppvfx. We start by showing that multiple samples drawn using the same \nystrom approximation are independent.
\begin{lemma}[restated \Cref{l:independent}]
Let $C\subseteq [n]$ be a random set variable with any
distribution. Suppose that  $S_1$ and $S_2$ are returned by two executions of \dppvfx, both using inputs constructed from the same $\L$
  and $\Lbh=\L_{\Ic,C}\L_C^+\L_{C,\Ic}$. Then $S_1$ and $S_2$ are
  (unconditionally) independent. 
\end{lemma}
\begin{proof}
Let $A$ and $B$ be two subsets of $[n]$ representing elementary events for
$S_1$ and $C$, respectively. Theorem \ref{t:exact} implies that
\begin{align*}
  \Pr(S_1\!=\!A \mid C\!=\!B) = \frac{\det(\L_A)}{\det(\I+\L)} = \Pr(S_1\!=\!A).
\end{align*}
Now, for any $A_1,A_2\subseteq[n]$ representing elementary events for
$S_1$ and $S_2$ we have that
\begin{align*}
  \Pr(S_1\!=\!A_1\wedge S_2\!=\!A_2) &= \sum_{B\in[n]}
  \Pr(S_1\!=\!A_1\wedge S_2\!=\!A_2\mid C\!=\!B)\,\Pr(C\!=\!B)
\\ & =\sum_{B\in[n]}\Pr(S_1\!=\!A_1\mid C\!=\!B)\Pr(S_2\!=\!A_2\mid
     C\!=\!B)\,\Pr(C\!=\!B)
\\ &=\Pr(S_1\!=\!A_1)\Pr(S_2\!=\!A_2)\sum_{B\in[n]}\Pr(C\!=\!B).
\end{align*}
Since $\sum_{B\in[n]}\Pr(C\!=\!B)=1$, we get that $S_1$ and $S_2$ are independent.
\end{proof}
We now bound the precompute cost, starting with the construction of the \nystrom approximation $\Lbh$.
\begin{lemma}[restated \Cref{lem:control-s-k}]
Let $\Lbh$ be constructed by sampling $m = \cO(k^3\log\frac n\delta)$
columns proportionally to their RLS. Then, with probability $1 - \delta$,
$\Lbh$ satisfies the assumption of \Cref{thm:fast-sampling-result}.
\end{lemma}
\begin{proof}
Let $\L=\B\B^\top$ and $\Lbh=\B\P\B^\top$ (where $\P$ is a projection
matrix).  Using algebraic manipulation, we can write 
\begin{align*}
s = \tr(\L-\Lbh+\Lbh(\I+\Lbh)^{-1}) = \tr(\B(\P\B^\transp\B\P + \I)^{-1}\B).
\end{align*}
The $\P\B^\transp\B\P$ matrix in the above expression been recently analyzed by \cite{calandriello_2019_coltgpucb}
in the context of RLS sampling who gave the following result that we use in the proof.
\begin{proposition}[{\citealp[Lemma~6]{calandriello_2019_coltgpucb}}]\label{prop:proj-A-accurate}
Let the projection matrix $\P$ be constructed by sampling $\cO(k\log(\frac n\delta)/\varepsilon^2)$ columns
proportionally to their RLS. Then,
\begin{align*}
(1-\varepsilon)(\B^\transp\B + \I)
\preceq \P\B^\transp\B\P + \I \preceq
(1+\varepsilon)(\B^\transp\B + \I).
\end{align*}
\end{proposition}
We will decompose the condition into two parts.
First we bound $k - s \leq 1/2$, and then
bound $z - \tr\big(\B\P(\I+\B^\top\B)^{-1}\P\B^\top\big) \leq 1/2$
(see the proof of \Cref{thm:fast-sampling-result} for more details).
It is easy to see that applying the construction from
\Cref{prop:proj-A-accurate} leads to the following bound on $s$,
\begin{align*}
s = \tr(\B(\P\B^\transp\B\P + \I)^{-1}\B)
\leq 
\frac{1}{1-\varepsilon}\tr(\B(\B^\transp\B + \I)^{-1}\B)
=\frac{1}{1-\varepsilon}k
=k + \frac{\varepsilon}{1-\varepsilon}k.
\end{align*}
Tuning $\varepsilon = 1/(2k+2)$ we obtain $s \leq k + 1/2$ and reordering gives us the desired accuracy result.
Similarly, we can invert the bound of \Cref{prop:proj-A-accurate} to obtain
\begin{align*}
(\B^\transp\B + \I)^{-1}
\preceq (1 +\varepsilon)(\P\B^\transp\B\P + \I)^{-1}
\end{align*}
and therefore
\begin{align*}
\tr\big(\B\P(\I+\B^\top\B)^{-1}\P\B^\top\big)
\leq (1+\varepsilon)
\tr\big(\B\P(\I+\P\B^\top\B\P)^{-1}\P\B^\top\big)
= (1 + \varepsilon)z
\leq (1+\varepsilon)k,
\end{align*}
where the last inequality is due to the fact that
$\Lbh \preceq \Lb$ since it is a \nystrom approximation
and that the operator $\tr\big(\Lb(\Lb + \I)^{-1}\big)$
is monotone. With the same $\varepsilon$ as before, we obtain
the bound. Summing the two $1/2$ bounds gives us the result.
\end{proof}
Finally, we show how to compute the remaining quantities needed for \dppvfx (\Cref{alg:main}).
\begin{lemma}[restated \Cref{l:computing}]
Given $\Lb$ and an arbitrary \nystrom approximation $\Lbh$ of rank $\Csize$, computing $l_i$, $s$, $z$, and $\Lbt$ requires $\cO(n\Csize^2 + \Csize^3)$ time.
\end{lemma}
\begin{proof}
Given the \nystrom set $C$, let us define the matrix $\Bb \triangleq \Lb_{\mathcal{I},C}\L_{C}^{+/2} \in \R^{n \times \Csize}$ such that $\Lbh = \Bb\Bb^\transp$. We also introduce $\Lbh_{\Csize} \triangleq \Bb^\transp\Bb$ to act as a $\R^{\Csize \times \Csize}$ counterpart to $\Lbh$. Denote with $\e_i$ the $i$-th indicator vector. Then,
exploiting the fact that $\Bb\Bb^\transp(\I+\Bb\Bb^\transp)^{-1} = \Bb(\I+\Bb^\transp\Bb)^{-1}\Bb^\transp$
for any matrix, we can compute $l_i$
as
\begin{align*}
l_i &= [\Lb - \Lbh + \Bb\Bb^\transp(\I+\Bb\Bb^\top)^{-1}]_{ii}
= [\Lb - \Lbh]_{ii} + \normsmall{(\I+\Lbh_{\Csize})^{-1/2}\Bb^\top\e_{i}}_2^2.
\end{align*}
Computationally, this means that we first need to compute $\Bb$, which takes
$\cO(\Csize^3)$ time to compute $\L_{C,C}^{+/2}$, and $\cO(n\Csize^2)$ time for the matrix multiplication.
Then, $[\Lbh]_{ii}$ is the $\ell_2$ norm of the $i$-th row of $\Bb$ which can be computed
in $n\Csize$ time. Similarly, $\normsmall{(\I+\Lbh_{\Csize})^{-1/2}\Bb^\transp\e_i}_2^2$
requires $\cO(\Csize^3 + n\Csize^2)$ time.
To compute $\sh,$ we simply sum $l_i$, while to compute $\st$ we first
compute the eigenvalues of $\Lbh_{\Csize}$,
$a_i=\lambda(\Lbh_{\Csize})_{i}$, in $\cO(\Csize^3)$ time,
and then compute $\st = \sum_i a_i/(a_i + 1)$.
We can also recycle the eigenvalues to precompute $\log\det(\I + \Lbh) = \log\det(\I + \Lbh_{m}) = \sum_{i} \log(a_i + 1)$.
\end{proof}

\section{Omitted proofs for the reduction to k-DPPs}\label{a:k-dpp}
In this section we present the proofs omitted from Section \ref{s:k-dpp}. Recall that our approach is based on the following rejection sampling strategy:
\begin{align*}
  \text{sample}\quad S_{\alpha}\sim \DPP(\alpha\L),\quad\text{accept
  if }|S_{\alpha}|=k. 
\end{align*}
First, we show the existence of the factor $\alpha^{\star}$ for which the rejection sampling is efficient.
\begin{theorem}[restated \Cref{lem:k-dpp-repeat}]
There exists constant $C>0$ such that for any rank $n$ PSD matrix $\L$ and
$k\in[n]$, there is $\alpha^\star > 0$ with the following property: if we
sample $S_{\alpha^\star}\sim\DPP(\alpha^\star \L)$, then
  \begin{align}
    \Pr(|S_{\alpha^\star}|=k)\geq \frac1{C\sqrt{k}}\cdot\label{eq:sup}
  \end{align}
\end{theorem}
\begin{proof}
W.l.o.g. assume that $\L$ is non-zero,
and remember $S_\alpha\sim\DPP(\alpha\L)$ with $k_\alpha=\E[\sizeS{\alpha}]$.
At a high level, the proof proceeds as follows. We first prove that
the probability that the size of the subset $\sizeS{\alpha}$ is equal to its \emph{mode}
$\mode_{\alpha}$, \ie $\Pr(\sizeS{\alpha} = \mode_{\alpha})$ is large enough.
Then we show that varying $\alpha$ can make $\mode_{\alpha} = k$ for any
$k$, and therefore we can find an $\alpha^\star$
s.t.~$\Pr(\sizeS{\alpha^\star} = \mode_{\alpha^\star}) = \Pr(\sizeS{\alpha^\star} = k)$
is large enough. In other words, rescaling $\DPP(\alpha\Lb)$ to
make sure that its mean $k_{\alpha}$ is close to $k$ is sufficient
to guarantee that $\sizeS{\alpha} = k$ with high enough probability.

Our starting point is a standard Chernoff bound for $\sizeS{\alpha}$.
\begin{proposition}[\citealp{dpp-concentration}]
  \label{t:concent}
 Given any PSD matrix $\L$, if $S_{\alpha}\sim\DPP(\alpha\L)$, then for
 any $a>0,$ we have
 \begin{align*}
   \Pr\Big( \big| \sizeS{\alpha} - \E[\sizeS{\alpha}]\big| \geq a\Big)\leq 5\exp\bigg(-\frac{a^2}{16(a+2\,\E[\sizeS{\alpha}])}\bigg)\cdot
 \end{align*}
\end{proposition}

Note that \Cref{t:concent} is not sufficiently strong by itself,
\ie if we tried to bound the distance $| \sizeS{\alpha} - \E[\sizeS{\alpha}]|$ to be smaller
than $1$ we would get vacuous bounds.
However, \Cref{t:concent} implies that
there is a constant $C>0$ independent of $\L$ such that 
$\Pr\big(|\sizeS{\alpha}-k_\alpha|\geq C\sqrt{k_\alpha}+1\big)\leq
\frac12$ for all $\alpha>0$. In particular, this means that the mode of $\sizeS{\alpha}$,
i.e. $\mode_\alpha=\argmax_i\Pr(\sizeS{\alpha}=i)$ satisfies
\begin{align}
  \Pr(\sizeS{\alpha}=\mode_\alpha)\geq \frac1{2C\sqrt{k_\alpha}}\sum_{i=-
  \lceil C\sqrt{k_\alpha}\rceil}^{\lceil C\sqrt{k_\alpha}\rceil}\Pr(\sizeS{\alpha}=k_\alpha+i)
  \geq \frac1{4C\sqrt{k_\alpha}}\cdot\label{eq:mode}
\end{align}
The distribution of $\sizeS{\alpha}$ is given by $\Pr(\sizeS{\alpha}=i)\propto e_i(\alpha\L)$, where $e_i(\cdot)$ is the
$i$th elementary symmetric polynomial of the eigenvalues of a
matrix. Denoting $\lambda_1,\dots,\lambda_n$ as the eigenvalues of
$\L$, we can express the elementary symmetric polynomials as the
coefficients of the following univariate polynomial with real
non-positive roots,
\begin{align*}
  \prod_{i=1}^n(x+\alpha\lambda_i) = \sum_{k=0}^nx^ke_{n-k}(\alpha\L).
\end{align*}
The non-negative coefficients of such a real-rooted polynomial form a
unimodal sequence (Lemma~1.1 in \cite{polynomials}),
i.e.,~$e_{0}(\alpha\L)\leq\dots\leq e_{\mode_\alpha}(\alpha\L)\geq\dots\geq
e_n(\alpha\L)$, with the mode (shared between no more than two positions $k,k+1$) being close to the mean $k_\alpha$:
$|\mode_\alpha-k_\alpha|\leq 1$ (Theorem 2.2 in \cite{polynomials}). Moreover, it is easy to see that $\mode_0=0$ 
and $\mode_\alpha=n$ for large enough $\alpha$, so since the sequence is
continuous w.r.t.~$\alpha$, for every $k\in[n]$ there
is an $\alpha^\star$ such that $\Pr(\sizeS{\alpha^\star}=k) = \Pr(\sizeS{\alpha^\star}=\mode_{\alpha^\star})$
(every $k$ can become one of the modes). In light of \eqref{eq:mode},
this means that 
\begin{align*}
   \Pr(\sizeS{\alpha^\star}=k) \geq \frac1{4C\sqrt{k_{\alpha^\star}}}\geq\frac1{4C\sqrt{k+1}}\CommaBin
\end{align*}
where the last inequality holds because $|k-k_{\alpha^\star}|\leq 1$.
\end{proof}
Finally, we show how to find $\alpha^\star$ efficiently.
\begin{lemma}
If $k \geq 1$ there is an algorithm that finds $\alpha^\star$ in
$\cO(n\cdot\poly(k))$ time.
\end{lemma}
\begin{proof}
In order to leverage \Cref{lem:k-dpp-repeat}, we need to find
an $\alpha^\star$ such that $k = \mode_{\alpha^\star}$, that is
such that the mode of $\DPP(\alpha^\star\Lb)$ is equal to $k$.
Unfortunately simple unimodality is not sufficient
to control $\mode_{\alpha^\star}$ when $\alpha^\star$ is perturbed,
as it happens during an approximate optimization of $\alpha$.
We will now characterize more in detail the distribution
of $\sizeS{\alpha}$.

In particular, $\sizeS{\alpha}$ can be defined as
the sum of Bernoullis $\sizeS{\alpha} = \sum_{i=1}^n b_{\alpha,i}$
each distributed according to $b_{\alpha,i} \sim \mathrm{Bernoulli}\left(\lambda_{i}(\alpha\Lb)/(1 + \lambda_{i}(\alpha\Lb))\right)$
\cite{dpp-independence}.
The sum of independent but not identically distributed Bernoullis
is a so-called Poisson binomial random variable \cite{hoeffding1956distribution}.
More importantly, the following result holds for Poisson binomial random variable.
\begin{proposition}[{\citealp[Thm.~4]{darroch1964distribution}}]\label{prop:pois-bin-dist}
Given a Poisson binomial r.v.\,$\sizeS{\alpha}$ with mean
$k_{\alpha}$, let $k \triangleq \lfloor k_{\alpha} \rfloor$. The mode $\mode_{\alpha}$ is
\begin{equation*}
\mode_{\alpha} = \begin{cases}
k & \quad\text{if } \quad k \leq k_{\alpha} < k + \frac{1}{k+2}\CommaBin\\
k \;\text{ or }\; k + 1 & \quad\text{if }\quad k + \frac{1}{k+2} \leq k_{\alpha} \leq k + 1 - \frac{1}{n - k + 1}\CommaBin\\
k + 1 & \quad\text{if }\quad k + 1 - \frac{1}{n - k + 1} < k_{\alpha} \leq k + 1.
\end{cases}
\end{equation*}
\end{proposition}
Therefore it is sufficient to find any constant $\alpha^\star$ that places $k_{\alpha^\star}$ in the interval $[k,k + \frac{1}{k+2})$.
Unfortunately, while the formula for $k_{\alpha} = \sum_{i=1}^n \lambda_{i}(\Lb)/(1/\alpha + \lambda_{i}(\Lb))$ is a unimodal function
of the eigenvalues of $\Lb$ which is easy to optimize,
the eigenvalues themselves are still very expensive to compute.
For efficiency, we can optimize it instead on the eigenvalues
of a \nystrom approximation $\Lbh$, but we have to be careful to control the error.
In particular, remember that $k_{\alpha} = \E[|S_{\alpha}|]$ when
$S \sim \DPP(\alpha\Lb)$, so given a \nystrom approximation
$\Lbh$ we can define  $s_{\alpha} \triangleq \tr(\alpha(\Lb - \Lbh) + \Lbh(\Lbh + \I/\alpha)^{-1})$
as a quantity analogous to $s$ from \dppvfx.
Then, we can strengthen \Cref{lem:control-s-k} as follows.
\begin{lemma}[crf.~\Cref{lem:control-s-k}]\label{cor:control-s-k-kdpp}
Let $\Lbh$ be constructed by sampling $m = \cO((k_{\alpha}/\varepsilon^2)\log(n/\delta))$
columns proportionally to their RLS. Then with probability $1 - \delta$
\begin{align*}
\frac{1}{1+\varepsilon}k_{\alpha} \leq s_{\alpha} \leq \frac{1}{1-\varepsilon}k_{\alpha}.
\end{align*}
\end{lemma}
\begin{proofof}{\Cref{cor:control-s-k-kdpp}}{}
{We simply apply the same reasoning of \Cref{lem:control-s-k} on both sides.}
\end{proofof}
Let $(1-\varepsilon)s_{\alpha^\star} = k$, with $\varepsilon$ that will be tuned shortly. Then
proving the first inequality to satisfy \Cref{prop:pois-bin-dist} is straightforward:
$k = (1 - \varepsilon)s_{\alpha^\star} \leq k_{\alpha^\star}$.
To satisfy the other side we upper bound
$k_{\alpha^\star}
\leq (1 + \varepsilon)s_{\alpha^\star}
= (1 - \varepsilon)s_{\alpha^\star} + 2\varepsilon s_{\alpha^\star}.$
We must now choose $\varepsilon$ such that $2\varepsilon s_{\alpha^\star} = 1/(k+3) < 1/(k+2)$.
Substituting, we obtain $\varepsilon = \frac{1}{2(k+3)s_{\alpha^\star}}\cdot$
Plugging this in the definition of $s_{\alpha^\star}$ we obtain
that $\alpha^\star$ must be optimized to satisfy
\begin{align*}
s_{\alpha^\star} = \frac{2k^2 + 6k + 1}{2k + 6}\CommaBin
\end{align*}
which we plug in the definition of $\varepsilon$ obtaining
our neccessary accuracy $\varepsilon = 1/(2k^2 + 6k + 1)$. Therefore,
sampling $m = \tcO(k_{\alpha^\star}k^4)$ columns gives us a $s_{\alpha}$ sufficiently
accurate to be optimized.
However, we still need to bound $k_{\alpha^\star}$, which we can do as follows
using \Cref{cor:control-s-k-kdpp} and $k \geq 1$
\begin{align*}
k_{\alpha^\star}
&\leq \left(1 + \frac{1}{2k^2 + 6k + 1}\right)s_{\alpha^\star}
\leq \left(1 + \frac{1}{9}\right)s_{\alpha^\star}
=\frac{10}{9}s_{\alpha^\star}\\
&\leq \frac{10}{9}\left(\frac{2k^2 + 6k + 1}{2k+6}\right)
= \frac{10}{9}\left(1 + \frac{1}{k(2k+6)}\right)k
= \frac{10}{9}\frac{9}{8}k
= \frac{5}{4}k.
\end{align*}
Therefore $m = \tcO(k_{\alpha^\star}k^4) \leq \tcO(k^5)$ suffices accuracy wise.
Moreover, since $s_{\alpha}$ is parametrized only in terms of the eigenvalues
of $\Lbh$, which can be found in $\tcO(nm^2 + m^3)$ time,
we can compute an $\alpha^\star$ such that
$s_{\alpha^\star} = \tfrac{2k^2 + 6k + 1}{2k + 6}$ in $\tcO(nk^{10} + k^{15})$
time, which guarantees
$k \leq k_{\alpha^\star} < k + \tfrac{1}{k+2}\cdot$
\end{proof}
Finally, note that 
these bounds on the accuracy of $\Lbh$ are extremely conservative. In practice,
it is much faster to try to optimize $\alpha^\star$ on a much coarser $\Lbh$ first,
\eg for $m = \cO(k_1)$, and only if this approach fails to increase the accuracy
of $\Lbh$.

\end{document}